\documentclass[10pt,twocolumn]{article}

\usepackage[utf8]{inputenc}
\usepackage[T1]{fontenc}
\usepackage{amsmath}
\usepackage{amssymb}
\usepackage{amsthm}
\usepackage{booktabs}
\usepackage{graphicx}
\usepackage{placeins}
\usepackage{hyperref}
\usepackage{geometry}
\usepackage{xcolor}
\usepackage{url}
\hypersetup{colorlinks=true,citecolor=blue,linkcolor=blue,urlcolor=blue}
\graphicspath{{figures/}}
\newtheorem{proposition}{Proposition}
\newtheorem{theorem}{Theorem}

\title{Iterative Erasure Count Is Not an Affine-Invariant Concept Dimension}
\author{%
Tingan Jin$^{1,*}$ \quad
Shuhang Dong$^{1,*}$ \quad
Haosong Li$^{2,*}$ \quad
Chung-Hsien Chou$^{3,\dagger}$\\[0.35em]
\small $^{1}$UCLA \quad $^{2}$Independent Researcher \quad
$^{3}$Cal Poly Pomona\\[0.25em]
\small $^{*}$Equal contribution. $^{\dagger}$Corresponding author.
}
\date{}

\begin{document}
\raggedbottom
\maketitle

\begin{abstract}
How many directions does a neural representation use to encode a concept? A
common operational answer repeatedly erases probe directions and reports the
stopping count or cumulative removed rank. We show that both quantities can
change under an invertible reparameterization that preserves all information, so
neither is intrinsically a concept dimension.

We distinguish model-defined population quantities---generating dimension,
sufficient linear dimension, and minimum guarding rank---from procedure-defined
quantities such as stopping count and cumulative edit rank. In a population
Gaussian construction, an invertible shear preserves the prediction problem and
all three population quantities, yet changes the cumulative Euclidean erasure
count from one to two. The separation holds for Moore--Penrose ordinary least
squares and every finite nonnegative ridge weight. For a two-output full-QR
procedure matching the algebra of our motivating video analysis, cumulative edit
rank similarly changes from two to the ambient dimension four.

Conversely, the complete cumulative metric-QR trajectory is affine-equivariant
when its positive-definite metric, probe, regularizer, and tie-breaking are
transported consistently; exact covariance is one corollary, not a canonical
semantic metric. In a known-rank finite-sample Adam/QR calibration, identity
mixing stops after one accepted update in all 20 large-sample runs, whereas each
tested shear $a\in\{.5,.75,1,1.25,2\}$ accepts at least two updates in all 20
runs. Controlled reparameterizations of frozen
V-JEPA2 features preserve rank-zero predictions yet alter later Euclidean
trajectories under practical optimization. These visual contact experiments are
stress tests, not estimates of contact dimension.

Iterative erasure therefore returns a procedure-relative estimand jointly
determined by representation geometry and the full measurement procedure, not a
semantic dimension by itself.
\end{abstract}

\section{Introduction}

Linear probes establish whether a label is accessible from a representation, but
they do not by themselves determine how that information is organized
\cite{alain2016understanding,hewitt2019designing,belinkov2021probing,
pimentel2020information}. A common next step is iterative null-space projection:
fit a linear classifier, remove its direction, refit, and count directions until
a new classifier fails \cite{ravfogel2020null}. The original INLP paper carefully
frames this as guardedness, but also describes protected-attribute subspaces as
spanning ``dozens to hundreds'' of orthogonal directions. Later applications go
further: recent work on video world models estimates feature dimensionality as
the number of orthogonal probes trainable before chance and interprets tens of
directions as distributed physical variables \cite{joseph2026physics}. Thus the
target of our critique is not merely reporting an algorithm output. It is using
that output to identify intrinsic dimensionality or distributedness. This does
not erase separate evidence from tuning geometry or held-out steering: those
experiments can still establish behavior of a declared native-basis subspace.
Our claim is specifically that the iterative count cannot supply a
coordinate-free dimension for that evidence.

That interpretation conflates several distinct quantities. A label may be
generated from one latent scalar while remaining predictable from many observed
coordinates because an invertible mixing matrix correlates those coordinates.
Conversely, a rank-one empirical cross-covariance can be canceled in one affine
edit without proving that one population-sufficient direction has been found.
The erasure path also depends on the feature metric, probe loss, regularizer,
covariance estimate, and stopping rule. These are properties of an analysis
procedure, not of a semantic concept alone.

Closed-form work already makes covariance central. Mean Projection and LEACE can
guard linear first-moment access with less collateral change than iterative
null-space projection
\cite{haghighatkhah2022single,belrose2023leace,dobrzeniecka2025improving}. We do
not propose another eraser, and one-step Euclidean non-equivariance is elementary.
Gauge-freedom work likewise shows that invertible reparameterizations preserve a
network function while changing Euclidean representation geometry
\cite{cain2026gauge}. Our sharper contribution is the discrete identification
result those observations do not provide: an exact integer-count change with
fixed sufficient and guarding rank, its full-QR counterpart, and an equivariance
theorem for the complete transported-metric trajectory. A finite reproduction of
the published QR stopping protocol and a map-specific fresh-attacker audit bridge
those results to practice. We ask:
\emph{what, if anything, can an iterative erasure count identify about concept
dimension?} Our answer is negative unless a metric and estimand are fixed in
advance.

We combine a population construction with a visual case study. The construction
has a known rank-one sufficient concept, independent nuisance variables, and a
family of invertible mixings with controlled condition number. It separates
coordinate-induced Euclidean redundancy from the known population guard. The
case study uses hand--object contact, a useful predicate for manipulation and
action understanding \cite{shan2020understanding,nguyen2026detecting}, in frozen
V-JEPA2 and DINOv2 representations \cite{mido2025vjepa,oquab2023dinov2}. Contact
is not offered as a pure semantic primitive: 100DOH contains object-configuration
shortcuts, and TouchMoment includes approach and action phase. Those limitations
make it a demanding audit target rather than evidence for a universal contact
mechanism.

The visual study is deliberately a case study rather than a confirmatory contact
claim. Validation-only layer sweeps select the frozen representations to audit;
all coordinate-stress conclusions are reported as conditional on those choices.
Optional native-basis channel analyses are confined to the supplement and do not
support the central identification result.

Our contributions are:
\begin{enumerate}
  \item We prove count non-identification for population ridge probes and give a
  full-QR two-output construction with cumulative edit rank two or four despite
  fixed sufficient dimension and guarding rank. We then reproduce the motivating
  finite Adam/MSE/QR recurrence and show held-out count changes on the same circular
  targets under invertible shears.
  \item We prove affine equivariance of the complete trajectory for any
  covariantly transported positive-definite metric, recover exact covariance as a
  corollary, and delimit the result under rank-deficient probe fitting.
  \item We apply predeclared anisotropic and orthogonal maps to unchanged V-JEPA2
  features and show that the trajectory changes under map-specific, per-rank fresh
  attackers trained and evaluated on source-disjoint official-training roles.
  \item We separate population dimension, moment rank, guarding rank, iteration
  count, and cumulative edit rank in a reproducible audit with no new official-test
  claims.
\end{enumerate}

\section{Related Work}

\paragraph{Probing and intervention.}
Linear probes reveal accessible information but can exploit confounds and do not
by themselves establish causal use. Probe conclusions depend on the attacker and
complexity control \cite{hewitt2019designing,belinkov2021probing,
pimentel2020information}. INLP repeatedly projects out classifier directions
\cite{ravfogel2020null}.
Amnesic probing uses the resulting cumulative subspace as a counterfactual
intervention and controls downstream effects against removal of the same number
of random directions \cite{elazar2021amnesic}. In high-dimensional NLI,
Rozanova et al. show that a small removed subspace relative to ambient dimension
and high-variance random-direction controls can obscure amnesic conclusions,
motivating a complementary mnestic analysis \cite{rozanova2023interventional}.
Mean Projection (MP) and LEACE instead provide one-shot lower-rank or
minimum-distortion guards
\cite{haghighatkhah2022single,belrose2023leace,dobrzeniecka2025improving}.
Haghighatkhah et al. find that after a targeted projection achieves linear
guarding, further INLP effects can resemble additional random projections;
later comparisons likewise document extra removed dimensions and collateral
change \cite{haghighatkhah2022single,dobrzeniecka2025improving}. These studies
use removed-direction counts to size interventions or controls, not to define an
affine-invariant concept dimension. LEACE characterizes linear
guardedness through feature--concept cross-covariance and provides a closed-form
minimum-distortion map with explicit affine structure. Spectral Attribute Removal
(SAL) removes the left singular subspace of feature--concept cross-covariance. For
a scalar binary concept, its full-rank subspace is the one-dimensional
mean-difference subspace \cite{shao2023sal,belrose2023leace}. Our novelty is not a covariance-aware
eraser or Proposition 1 alone. LEACE supplies a one-shot minimum-distortion guard.
We instead give an exact integer-count change, extend it to full-QR multivariate
removal, characterize equivariance of the complete transported-metric trajectory,
and separate residual access caused by finite estimation from population rank.
Exact covariance is one corollary of the metric theorem, not a new eraser. Minimax
erasure such as RLACE
answers a different, explicitly attacker-relative optimization problem
\cite{ravfogel2022rlace}. It reinforces rather than removes the need to declare
the estimand.

\paragraph{Concept geometry and non-identifiability.}
Concept Activation Vectors operationalize human concepts as feature-space
directions in vision models \cite{kim2018tcav}. Such directions can be useful in a
fixed representation, but an invertible map preserves information and all linear
scores while changing Euclidean angles. This is one instance of a broader
identification problem: representation factors and their geometry are not
recoverable from observational fit without assumptions or inductive bias
\cite{locatello2019challenging}. Under stated diversity conditions, Roeder et al.
prove that a broad discriminative model family is identifiable only up to
invertible linear transformations \cite{roeder2021linear}. Thus a quantity claimed
to be intrinsic across equivalent learned representations must either survive
this ambiguity or declare additional geometry. Cain formalizes the related
gauge-freedom view and
shows that Euclidean similarity can change under function-preserving invertible
maps \cite{cain2026gauge}. Our claim is narrower and discrete: for a specified
iterative eraser, the exact integer population stopping count changes while both
sufficient dimension and minimum guarding rank remain fixed.

\begin{table*}[t]
\caption{Claim boundary and novelty. The paper revises a dimensionality
interpretation; it does not invalidate every result produced alongside an erasure
count.}
\label{tab:claim-boundary}
\centering
\scriptsize
\begin{tabular}{@{}p{0.14\textwidth}p{0.24\textwidth}p{0.27\textwidth}p{0.24\textwidth}@{}}
\toprule
Prior object & What remains valid & What must be reinterpreted & What is new here \\
\midrule
INLP \cite{ravfogel2020null} & Linear guardedness and the edited representation under a declared native-basis procedure & Iteration count is not thereby an intrinsic protected-attribute dimension & Exact count-one/count-two separation with fixed sufficient and guarding rank \\
MP, SAL, LEACE \cite{haghighatkhah2022single,shao2023sal,belrose2023leace} & One-shot first-moment or affine-linear guarding with an explicit distortion geometry & Covariance is a chosen geometry, not a universally semantic one & Complete transported-metric trajectory theorem and rank-deficient scope \\
Gauge freedom \cite{cain2026gauge} & Euclidean similarities can change under function-preserving invertible maps & General metric dependence alone does not state that a discrete stopping count changes & Scalar and full-QR integer-count counterexamples \\
Video direction study \cite{joseph2026physics} & Layerwise access, tuning structure, and held-out steering remain native-basis evidence & Orthogonal-probe count supports only procedure-relative, not affine-invariant, distributedness & Matched circular Adam/MSE/QR finite stopping stress test \\
\bottomrule
\end{tabular}
\end{table*}

\paragraph{Frozen visual representations.}
Predictive and masked objectives learn transferable image and video features
\cite{wei2022masked,assran2023selfsupervised,bardes2024revisiting,mido2025vjepa}.
V-JEPA2 predicts latent video targets, while DINOv2 provides a strong image-only
comparison \cite{mido2025vjepa,oquab2023dinov2}. Prior egocentric work studies
actions, objects, and hand state \cite{ragusa2023stillfast,shiota2024egocentric,
yue2023egopca}. We study the stability of a geometric interpretation rather than
claiming a new contact detector.

\section{Rank Estimands and Affine Dependence}

\subsection{Five Quantities Called ``Dimension''}

Let $X\in\mathbb{R}^d$ be a representation and $Y$ a concept label. Table
\ref{tab:estimands} separates quantities that need not agree. Generating and
sufficient dimensions require assumptions about a latent population model.
Cross-covariance rank is a first-moment property. For a scalar encoding
$\phi(Y)$, it is at most one.
Iterative erasure count is the number of noncollapsed sequential probe updates
before a declared stopping rule fires; cumulative edit rank is
$\operatorname{rank}(I-T_k)$ after update $k$. When $T_k$ is an idempotent
projector, this equals the dimension of its removed subspace. For a nonprojector
recurrence it is only the rank of the cumulative edit operator, not a removed-
subspace dimension. Count and edit rank coincide only when every update adds one
independent direction. Linear guarding rank is the minimum
intervention rank needed to satisfy a specified guardedness criterion.
Formally, for a declared family of label distributions, define sufficient linear
dimension as the minimum $r$ for which some $B\in\mathbb{R}^{d\times r}$ satisfies
$Y\perp X\mid X B$. For a declared attacker class $\mathcal F$, a map $T$ is
$\mathcal F$-guarding at risk $R_0$ when
$\inf_{f\in\mathcal F}R(f(XT),Y)\geq R_0$. Linear guarding rank is the minimum
$\operatorname{rank}(I-T)$ satisfying that criterion. These definitions depend on
the population law, attacker class, loss, and baseline risk.

\begin{table}[t]
\caption{Distinct estimands that should not be reported under one name.}
\label{tab:estimands}
\centering
\scriptsize
\begin{tabular}{@{}p{0.24\columnwidth}p{0.70\columnwidth}@{}}
\toprule
Quantity & Definition or dependency \\
\midrule
Generating dimension & Number of latent variables used by the data-generating label rule \\
Sufficient linear dimension & Minimum population linear subspace sufficient for $Y$ under a stated model \\
Cross-covariance rank & $\operatorname{rank}\operatorname{Cov}(X,\phi(Y))$ for a declared encoding $\phi$, capturing only a first moment \\
Iterative erasure count & Stopping time; report cumulative edit rank separately for multivariate updates \\
Linear guarding rank & Minimum linear edit rank subject to a stated guardedness condition \\
\bottomrule
\end{tabular}
\end{table}

Guardedness itself also needs a declared strength. Zero empirical
$\widehat{\operatorname{Cov}}(X,Y)$, or equal empirical class means for binary
$Y$, is a sample first-moment condition. Failure of one fitted logistic or SVM
attacker is objective- and sample-relative. Failure of every affine linear
classifier is a class-wide property, while statistical independence of edited
features and $Y$ is stronger still. Mean Projection, SAL, and LEACE primarily
target first-moment or affine-linear guardedness. Our population proposition uses
full independence. We treat finite-attacker failure only as evidence about the
stated protocol, never as proof of independence.

\begin{figure}[t]
\centering
\fbox{\begin{minipage}{0.93\columnwidth}
\small
\textbf{Algorithm 1: Declared cumulative metric-QR erasure.}
Given centered features $X$, $G\succ0$ fixed from the intact representation, a probe rule, and a stopping rule,
set $U_0=[\,]$ and $T_0=I$. For $k=1,2,\ldots$:
\begin{enumerate}
\item Fit coefficient block $B_k=\operatorname{Probe}(XT_{k-1},Y)$.
\item Residualize $B_k$ against $U_{k-1}$ and metric-QR it to a
$G$-orthonormal new block $Q_k$.
\item If $Q_k$ is empty or the declared held-out rule fires, stop; otherwise set
$U_k=[U_{k-1},Q_k]$ and $T_k=I-U_kU_k^\top G$.
\item Record iteration count and $\operatorname{rank}(I-T_k)$ separately.
\end{enumerate}
Every metric, probe, regularizer, residualization rule, numerical tolerance, and
stopping rule is part of the estimand. The metric is not recomputed after editing.
\end{minipage}}
\par\smallskip
\footnotesize The transported-metric theorem covers this cumulative procedure.
The finite circular reproduction instead follows the motivating paper's literal
sequential projector recurrence and is labeled separately.
\end{figure}

\subsection{Euclidean Erasure Is Not Affine Invariant}

For a row feature $x$ and invertible $A$, write $z=xA$. A coefficient $w$ in the
original coordinates has equivalent transformed coefficient $w_z=A^{-1}w$ because
$zw_z=xw$. Projecting Euclideanly in transformed coordinates and mapping back gives
\begin{equation}
  x' = xA\left(I-\frac{w_z w_z^\top}{w_z^\top w_z}\right)A^{-1}.
  \label{eq:affine-projection}
\end{equation}
Equation \ref{eq:affine-projection} generally differs from
$x(I-ww^\top/(w^\top w))$ unless $A$ preserves the relevant Euclidean metric.
Thus an invertible reparameterization preserves all pre-edit linear predictions
but changes the erased subspace and every later refitted direction.

\begin{proposition}[Coordinate dependence]
For $d\geq2$ and nonzero $w$, there exists an invertible $A$ for which the
mapped-back Euclidean edit in Equation \ref{eq:affine-projection} differs from the
Euclidean edit in the original coordinates, although $xw=(xA)(A^{-1}w)$ for every
$x$.
\end{proposition}
\begin{proof}
Choose a basis with $w=e_1$ and let the leading $2\times2$ block of $A$ be
$\left[\begin{smallmatrix}1&0\\a&1\end{smallmatrix}\right]$ with $a\neq0$. Then
$A^{-1}w=(1,-a)^\top$, and direct substitution into Equation
\ref{eq:affine-projection} produces off-diagonal terms absent from
$I-e_1e_1^\top$. Score equivalence follows from invertibility. Rotating this basis
handles arbitrary nonzero $w$.
\end{proof}

This observation also clarifies the regularization confound. Fitting an isotropic
$\ell_2$ probe independently in $x$ and $z$ changes the prior. The transformed
penalty equivalent to $\|w\|_2^2$ is quadratic, not a newly selected scalar
$C\|w_z\|_2^2$. In our metric comparison, every attacker is therefore fit in the
same train-standardized original coordinates with one validation-selected $C$.
Whitening is used only to define an intervention, which is then mapped back. Rank
zero agrees to numerical precision (maximum absolute discrepancy below
$7\times10^{-14}$).

Here $G$ is a positive-definite quadratic form on probe coefficients, equivalently
on the dual representation space. Because coefficients transform as
$w_Z=A^{-1}w_X$, preserving $w^\top G w$ requires the congruence
$G_Z=A^\top G_XA$. A metric acting instead on primal row features would obey the
inverse-congruence rule.

\begin{proposition}[Complete transported-metric affine equivariance]
\label{prop:metric-equivariance}
Let $G_X\succ0$ be any coefficient metric fixed from the intact representation. At iteration $k$,
let the probe return a coefficient block $B_k$ (one column for a scalar target),
append its metric-orthogonalized column space to a $G_X$-orthonormal basis $U_k$,
and edit by $T_k=I-U_kU_k^\top G_X$. For $Z=XA$ with invertible $A$, transport
the metric as $G_Z=A^\top G_XA$. If the probe is equivariant
($B_k^Z=A^{-1}B_k$), regularization is transported under $A$, and any
non-subspace tie is resolved equivariantly, then for every $k$,
$Z T_k^Z=(X T_k)A$. Scores, cumulative edit rank, and every score-based stopping
count are identical in both parameterizations.
\end{proposition}
\begin{proof}
If $U_k^Z=A^{-1}U_k$, then
$(U_k^Z)^\top G_ZU_k^Z=U_k^\top G_X U_k=I$ and
\begin{align}
 A T_k^Z
 &=A\left[I-A^{-1}U_kU_k^\top A^{-\top}A^\top G_X A\right] \\
 &=\left[I-U_kU_k^\top G_X\right]A=T_kA.
\end{align}
The claim follows by induction: the edited features correspond by $A$, probe
equivariance maps the next coefficient block by $A^{-1}$, and metric
Gram--Schmidt preserves the basis relation. Transporting an original penalty
$\|w\|_2^2$ means using $w_Z^\top A^\top A w_Z$, not an isotropic penalty in
$Z$.
\end{proof}

Exact covariance is a natural corollary: setting $G_X=\Sigma_X$ gives
$G_Z=\Sigma_Z=A^\top\Sigma_XA$. Proposition~\ref{prop:metric-equivariance} does
not privilege covariance as a semantic geometry; any fixed positive-definite
metric transported by the same tensor law obeys the result. OAS shrinkage toward
identity \cite{chen2010shrinkage} and eigenvalue flooring generally break it because their covariance maps
$f$ do not satisfy $f(A^\top\Sigma A)=A^\top f(\Sigma)A$.

The probe assumption needs care after the first edit, when ambient features are
rank deficient. Strictly convex least-squares or logistic objectives with an
unpenalized intercept and a positive-definite quadratic penalty transported by
$A$ retain a unique equivariant coefficient even after rank loss. Unregularized
OLS is equivariant at a full-rank intact iteration. In contrast, rank-deficient
Moore--Penrose OLS selects the minimum \emph{Euclidean}-norm ambient coefficient;
that selector is generally not equivariant under a nonorthogonal $A$ and is not
covered by Proposition~\ref{prop:metric-equivariance}. Theorem
\ref{thm:full-qr-count} uses Moore--Penrose only to define one displayed-coordinate
procedure and makes no affine-equivariance claim. Our empirical unshrunk-covariance
check instead uses a strictly regularized probe after mapping every edited feature
back to the common original coordinates, so no independent pseudoinverse selector
is compared across parameterizations.

For a multivariate probe, removing the entire column space of $B_k$ makes the
update independent of the particular QR basis. Extracting one vector from a tied
singular subspace is not basis-independent and requires an additional declared
rule. Proposition~\ref{prop:metric-equivariance} does not make that choice
canonical.

The statement extends from $Z=XA$ to a full affine change $Z=XA+b$ by tracking
the mean, applying each edit to centered features, and transforming the intercept:
$\mu_Z=\mu_XA+b$ and $c_Z=c_X-bA^{-1}w$. The metric proof is then unchanged.
When an exact covariance metric is singular, the corollary applies only after a common
support subspace has been fixed. A pseudoinverse, eigenvalue floor, or shrinkage
rule adds an estimator-dependent extension outside that support and is not covered
by the equivariance claim.

\subsection{Population Rank-One Construction}

Let $H=(S,N)\sim\mathcal{N}(0,I_d)$, with scalar concept coordinate $S$ and
independent nuisance $N$. Define $Y=\mathbb{1}[S\geq0]$ and observed features
$X=HA$ for invertible $A$. The generating and minimum sufficient linear dimensions
are one. Let $e=(1,0,\ldots,0)^\top$. Moreover,
\begin{equation}
  T_A=A^{-1}(I-ee^\top)A
\end{equation}
has $\operatorname{rank}(I-T_A)=1$, and
$XT_A=H(I-ee^\top)A$ is independent of $Y$. Hence the population minimum guarding
rank is exactly one. Varying $A$ changes no labels, latent variables, or sufficient
dimension.

\begin{proposition}[Known minimum guard]
For the construction above, the minimum rank of a linear edit that makes
the edited representation independent of $Y$ is one.
\end{proposition}
\begin{proof}
$T_A$ gives a rank-one upper bound and removes $S$ exactly. A rank-zero edit is the
identity. Because $A$ is invertible, $S=XA^{-1}e$ remains a deterministic
linear function of $X$, and $Y=\mathbb{1}[S\geq0]$ is nonconstant. Thus rank zero
cannot guard $Y$, establishing the matching lower bound.
\end{proof}

Non-equivariance of one update does not alone prove that a stopping count changes.
The following population separation does.

\begin{theorem}[Cumulative-QR ridge count non-identification]
\label{thm:ridge-count}
Let $S,N\overset{\mathrm{iid}}{\sim}\mathcal N(0,1)$,
$L=\operatorname{sign}(S)$, and $X_a=(S+aN,N)$ for $a\in\mathbb R$. Maintain a
Euclidean-orthonormal cumulative basis $U_k$ and projector
$P_k=I-U_kU_k^\top$, beginning with $P_0=I$. At iteration $k$, fit population
ridge least squares on $X_aP_{k-1}$, residualize its ambient coefficient as
$q_k=P_{k-1}w_k$, append $q_k/\lVert q_k\rVert_2$ when nonzero, and stop when the
population coefficient is zero. Fix any finite ridge weight $\lambda\geq0$.
At $\lambda=0$, every rank-deficient fit uses the Moore--Penrose minimum-norm
coefficient. For $\lambda>0$, the penalty may be either isotropic in each displayed
coordinate system or transported from $X_0$ under $X_a=X_0A_a$.
The iteration count and cumulative edit rank are one for $a=0$ and two for
every $a\neq0$, although every $X_a$ has sufficient linear dimension one and
minimum linear guarding rank one.
\end{theorem}
\begin{proof}
Writing $c=\mathbb E[S\operatorname{sign}(S)]=\sqrt{2/\pi}$ gives
$\operatorname{Cov}(X_a,L)=c(1,0)^\top$. The ridge coefficient is proportional to
\begin{equation}
 [\operatorname{Cov}(X_a)+\lambda I]^{-1}\operatorname{Cov}(X_a,L)
 \propto(1+\lambda,-a)^\top .
\end{equation}
This is the isotropic-penalty solution. Writing
$A_a=\left[\begin{smallmatrix}1&0\\a&1\end{smallmatrix}\right]$, the transported
penalty is $w_a^\top A_a^\top A_a w_a$. Its solution is proportional to
$A_a^{-1}(1,0)^\top=(1,-a)^\top$. Both solutions are parallel to the
feature--label cross-covariance vector $(1,0)^\top$ only when $a=0$. At
$\lambda=0$, the isotropic and transported solutions themselves coincide for
every $a$. For $a=0$, this common direction equals the cross-covariance direction
and one edit leaves independent $N$. For $a\neq0$, neither fitted direction is
parallel to the cross-covariance vector $(1,0)^\top$, so
the first cumulative Euclidean removal leaves nonzero feature--label
cross-covariance in a retained line spanned by a unit vector $r$. Write
$v=\operatorname{Var}(X_ar)>0$, $g=\operatorname{Cov}(X_ar,L)\neq0$, and let
$\Gamma\succ0$ be either the identity or transported penalty matrix. For
$\lambda>0$, the second ambient coefficient satisfies
\begin{align}
 w_2 &=g(\lambda\Gamma+vrr^\top)^{-1}r, \\
 r^\top w_2 &=g\,r^\top(\lambda\Gamma+vrr^\top)^{-1}r\neq0.
\end{align}
Hence its cumulative residual $q_2=P_1w_2=(r^\top w_2)r$ is nonzero and spans
the retained support. At $\lambda=0$, the Moore--Penrose solution gives the same
conclusion, $q_2=(g/v)r$. Appending this residual exhausts the two-dimensional
representation, so the next coefficient is zero.
The rank-one guard
$(S+aN,N)\mapsto(0,N)$ establishes the common guarding rank. For fixed
$a\neq0$ and finite $\lambda$, the residual coefficient has positive norm, so the
same separation holds for a coefficient-based rule that accepts update $k$ when
$\lVert q_k\rVert_2>\tau$, provided $\tau$ is below the three relevant nonzero
norms $\lVert q_1(a=0)\rVert_2$, $\lVert q_1(a)\rVert_2$, and
$\lVert q_2(a)\rVert_2$. This statement does not cover a score- or risk-based
stopping threshold.
\end{proof}

The exact count jump should not be confused with a uniformly large practical
effect. Its magnitude is explicit. Let
$D_\lambda^2=(1+\lambda)^2+a^2$. Under isotropic ridge, the retained unit line
after the first edit is $r_\lambda=(a,1+\lambda)^\top/D_\lambda$, with
\begin{align}
 g_\lambda&=\operatorname{Cov}(X_ar_\lambda,L)
   =\frac{ca}{D_\lambda}, \\
 v_\lambda&=\operatorname{Var}(X_ar_\lambda)
   =\frac{a^2+(a^2+1+\lambda)^2}{D_\lambda^2}.
\end{align}
Because $\operatorname{Var}(L)=1$, the best retained scalar linear predictor has
\begin{align}
 R_{\mathrm{res}}^2
 &=\frac{g_\lambda^2}{v_\lambda}
 =\frac{(2/\pi)a^2}{a^2+(a^2+1+\lambda)^2}, \\
 \lVert q_2\rVert_2
 &=\frac{|g_\lambda|}{v_\lambda+\lambda}
 =\frac{c|a|D_\lambda}
 {a^2+(a^2+1+\lambda)^2+\lambda D_\lambda^2}.
 \label{eq:shear-residual}
\end{align}
For transported ridge, replace these by
$R_{\mathrm{res}}^2=(2/\pi)a^2/[a^2+(1+a^2)^2]$ and
$\lVert q_2\rVert_2=c|a|\sqrt{1+a^2}/
[a^2+(1+\lambda)(1+a^2)^2]$.
Both are positive for every $a\neq0$ but approach zero continuously as
$a\rightarrow0$; the coefficient norm also shrinks with strong ridge. Thus a
finite score threshold can legitimately return one even where exact population
coefficient stopping returns two. The finite experiment below measures that
bridge rather than hiding it.
The norm in Equation~\ref{eq:shear-residual} is specifically the cumulative
residualized coefficient $q_2$, not the raw ambient coefficient $w_2$.

The cumulative residualization is essential. Literal sequential multiplication
by null projectors of the raw ambient coefficients need not terminate in two
steps under a transported anisotropic penalty and can revisit an earlier support;
that different algorithm is not covered by Theorem~\ref{thm:ridge-count}.

The motivating video analysis fits a two-output regressor and removes the full QR
basis of its weight matrix. The next result matches that algebraic procedure.

\begin{theorem}[Full-QR multivariate count non-identification]
\label{thm:full-qr-count}
Let $S,N\overset{\mathrm{iid}}{\sim}\mathcal N(0,I_r)$, let $Y=S$, and let
$X_a=(S+aN,N)\in\mathbb R^{2r}$. At each iteration, fit population multivariate
least squares, use the Moore--Penrose solution if the edited covariance is
singular, remove an orthonormal basis for the full coefficient column space by QR,
and stop at zero feature--target cross-covariance. For $a=0$, the iteration count
is one and cumulative edit rank is $r$. For every $a\neq0$, the iteration count
is two and cumulative edit rank is $2r$, the ambient dimension. In both cases the
cumulative map is an orthogonal projector, so this edit rank also equals its
removed-subspace dimension. Yet the
sufficient linear dimension and minimum
independence-guarding rank are $r$ for every $a$.
\end{theorem}
\begin{proof}
The coefficient block $B_\star=[I_r,-aI_r]^\top$ satisfies
$X_aB_\star=S=Y$, giving sufficient linear dimension at most $r$. For the lower
bound, let $Z=X_aD\in\mathbb R^m$ be any linear statistic sufficient for $Y$.
Because $Y$ is a deterministic function of $X_a$, conditional independence
$Y\perp X_a\mid Z$ requires $\operatorname{Cov}(S\mid Z)=0$. Joint Gaussianity
gives
\[
 \operatorname{Cov}(S\mid Z)
 =I_r-\operatorname{Cov}(S,Z)\operatorname{Var}(Z)^+
       \operatorname{Cov}(Z,S).
\]
The subtracted term has rank at most $m$; equaling $I_r$ therefore requires
$m\geq r$. Hence the sufficient linear dimension is exactly $r$.

The first coefficient block is
$B=[I_r,-aI_r]^\top$. For $a=0$, its removal leaves only $N$, which is independent
of $Y$. For $a\neq0$, the orthogonal complement of $\operatorname{col}(B)$ is
spanned by $C=[aI_r,I_r]^\top$. The retained coordinates are proportional to
$X_aC=aS+(1+a^2)N$, whose cross-covariance with $Y$ is $aI_r$ and therefore has
rank $r$. More explicitly, with
$P_C=C(C^\top C)^{-1}C^\top$, the edited covariance
$\widetilde\Sigma=P_C\operatorname{Cov}(X_a)P_C$ is positive definite on
$\operatorname{col}(C)$, and
$\operatorname{Cov}(X_aP_C,Y)=aC/(1+a^2)$ has rank $r$. The Moore--Penrose OLS
block $\widetilde\Sigma^+\operatorname{Cov}(X_aP_C,Y)$ therefore has rank $r$,
lies in $\operatorname{col}(C)$, and thus spans all of $\operatorname{col}(C)$.
The second QR removal exhausts the representation.

An edit that sets the $S$ component to zero gives a rank-$r$ independent guard.
Conversely, independence in this jointly Gaussian model requires zero
feature--target cross-covariance. If $T$ guards $Y$, then
$(I-T)^\top\operatorname{Cov}(X_a,Y)=\operatorname{Cov}(X_a,Y)$, so
$\operatorname{rank}(I-T)\geq\operatorname{rank}\operatorname{Cov}(X_a,Y)=r$.
Thus the minimum guarding rank is exactly $r$.
\end{proof}

For the circular-regression output width $r=2$, Theorem~\ref{thm:full-qr-count}
changes the reported QR direction count from two to four. It matches the
two-output MSE and full-QR removal rule, but not a particular angular data
distribution or finite held-out stopping threshold.

\subsection{Finite Circular-QR Bridge}

We therefore implement the motivating finite procedure on a controlled angular
target. Draw $\theta\sim\operatorname{Unif}[-\pi,\pi]$, set
$S=(\sin\theta,\cos\theta)$, draw $N\sim\mathcal N(0,.5I_2)$ to match the
per-coordinate variance of $S$, and form $X_a=(S+aN,N)$. The same latent draws,
labels, and probe initialization are reused across shears. Every $X_a$ has a
two-dimensional sufficient statistic and an explicit rank-two guard that removes
$S$ while retaining $N$.

Following the published specification \cite{joseph2026physics}, each fresh linear
probe predicts $(\sin\theta,\cos\theta)$ with MSE and Adam
\cite{kingma2015adam} for 100 epochs at
learning rate $10^{-3}$ and weight decay $10^{-4}$. An 80/20 split is fixed for the
entire sequence. We apply the literal recurrence
$X^{(k+1)}=X^{(k)}(I-Q_kQ_k^\top)$, where $Q_k$ is the full QR basis of the
two-column coefficient block. An update is accepted only when held-out
$R^2\geq.1$ and circular MAE $\leq80^\circ$; $K$ counts accepted updates before
the first failing probe. The source does not report minibatch size, so the primary
run declares 128 and repeats the decisive $n=4{,}000$, $a\in\{0,1\}$ comparison
at 64 and 256. We record both $\operatorname{rank}(I-T_k)$ and
$\operatorname{rank}(T_k)$ rather than equating $2K$ with independent removed
rank. We call the former cumulative edit rank because the literal $T_k$ need not
be idempotent. Coefficient QR uses absolute tolerance $10^{-8}$, and both reported
map ranks use absolute singular-value tolerance $10^{-6}$. Twenty seeds are used
per primary cell, with an eight-update audit cap.

\begin{table*}[t]
\caption{Theorems and experiments analyze related but distinct declared
procedures. No row silently inherits guarantees from another.}
\label{tab:procedure-map}
\centering
\scriptsize
\begin{tabular}{@{}p{0.18\textwidth}p{0.26\textwidth}p{0.22\textwidth}p{0.23\textwidth}@{}}
\toprule
Object & Update and probe & Stop & What it establishes \\
\midrule
Theorem~\ref{thm:ridge-count} & Cumulative Euclidean Gram--Schmidt; scalar population ridge & Zero population coefficient & Exact count 1 versus 2; excludes raw sequential projectors \\
Theorem~\ref{thm:full-qr-count} & Full coefficient-column QR; population OLS with declared Moore--Penrose selector & Zero feature--target cross-covariance & Cumulative edit rank $r$ versus $2r$; no equivariance claim \\
Finite circular bridge & Literal sequential QR; Adam/MSE, two outputs, fixed held-out split & Published $R^2$/circular-MAE thresholds & Practical stopping sensitivity and count/edit-rank separation \\
Visual case study & Cumulative scalar Gram--Schmidt; LBFGS logistic eraser and fresh attacker & Fixed audited ranks; descriptive validation stopping secondary & Procedure dependence on unchanged frozen visual features \\
\bottomrule
\end{tabular}
\end{table*}

We verify Theorem~\ref{thm:ridge-count} for
$\lambda\in\{0,.1,1,10\}$. A complementary population control
uses $H\sim\mathcal N(0,I_8)$, $X=HA$, and continuous target
$Y=(H_1,.7H_2)$. Its sufficient dimension and minimum guarding rank are two. At
each step, multivariate least squares is refit and the leading left singular vector
of its coefficient matrix is removed, with unequal target scales fixing tie-breaking.
Stopping requires zero feature--target cross-covariance. Across 50 predeclared
Gaussian-QR orientations per condition, singular values of $A$ are geometrically
spaced between $\kappa^{-1/2}$ and $\kappa^{1/2}$. The scalar theorem verification
uses relative cross-covariance, coefficient, and numerical-rank tolerance
$10^{-10}$. The ambient-eight control uses relative Frobenius cross-covariance
and direction-collapse tolerance $10^{-9}$ and Moore--Penrose cutoff $10^{-9}$.

\section{Exploratory Visual Case-Study Protocol}

The central non-identification claim is established by the population and
controlled finite constructions above. The visual analyses ask only whether the
same procedural sensitivity is visible in frozen features; they do not estimate a
contact dimension or provide an untouched confirmatory benchmark.

\subsection{Tasks and Representations}

100DOH labels an image positive when any annotated hand has contact state 3/4 and
negative only when all hands have state 0 \cite{shan2020understanding}. Ambiguous
states 1/2 are excluded. The balanced audit subset has 1,600/200/200
train/validation/test images with released split boundaries and supports matched
V-JEPA2/DINOv2 audits. A separate maximal balanced V-JEPA2 audit uses every 10,645
eligible negative and an equal number of positives: 17,206/2,126/1,958 examples
and 13,081 source videos, with no video crossing splits. TouchMoment is constructed
from HOI4D and the egocentric subset of TACO
\cite{liu2022hoi4d,liu2024taco,nguyen2026detecting}. Positives end at a touch
onset. Negatives end 0.5 seconds earlier in the same video, with hand visibility
required and other touches excluded from each causal window. It contains
11,564/618/2,588 examples and 1,294 test
pairs from 647 videos. Videos do not cross splits.

All 100DOH splitting and resampling groups by source video. TouchMoment groups by
source video, preserving each onset/pre-onset pair. Unless an analysis reports
split-seed quantiles, uncertainty uses 1,000 group-bootstrap replicates. No row-
or frame-level bootstrap is used.

We freeze \path{facebook/vjepa2-vitl-fpc64-256}, mean-pool hidden tokens, and use
validation-selected layers 21 (100DOH) and 18 (TouchMoment). DINOv2-Base uses
candidate images and layers 12 and 11. Feature dimensions are 1,024 and 768.
AdamW probes with seeds 17/23/42 are used only for the initial layer sweep. All
rank audits use deterministic LBFGS logistic probes or LinearSVC.

The primary label is contact. Predicate, nuisance, temporal, random-subspace, and
stability audits are retained in the artifact but are secondary to identification.

\subsection{Metric, Estimator, and Stopping Controls}

Features are standardized using official-training statistics. We compare
Euclidean intervention geometry with empirical, Oracle Approximating Shrinkage
(OAS), and Ledoit--Wolf covariance maps
\cite{chen2010shrinkage,ledoit2004well}. Eigenvalues are floored at
$10^{-4}\lambda_{\max}$ where needed. Unshrunk empirical covariance, when full rank,
uses neither shrinkage nor a floor. Unlike shrinkage toward identity or flooring,
it is affine-equivariant. We report the full eigenspectrum, effective and stable
rank, pre-floor positive-spectrum condition number, post-floor condition number,
floored count, and direction alignment. Five independent stratified covariance
subsamples are drawn at every non-full sample size.

The metric-aware iteration is fixed explicitly. Let $W=I$ for Euclidean geometry
or $W=\widetilde\Sigma^{-1/2}$ for a covariance metric, where
$\widetilde\Sigma$ and its eigenvalue floor are estimated once from intact
training features. Set $Z_0=XW$ and $U_0$ empty. At iteration $k$, fit coefficient
$w_k$ on $X_k=Z_kW^{-1}$ with one intact validation-selected $C$, map it to metric
coordinates as $a_k=W^{-1}w_k$, and compute
\begin{align}
 q_k &=(I-U_{k-1}U_{k-1}^{\top})a_k, &
 u_k &=q_k/\lVert q_k\rVert_2, \\
 U_k &=[U_{k-1},u_k], &
 Z_k &=Z_0(I-U_kU_k^{\top}).
 \label{eq:iteration}
\end{align}
We stop if $q_k$ numerically collapses. Covariance is never recomputed from edited
features, so singular edited covariances do not enter later steps. The cumulative
original-coordinate map is $T_k=W(I-U_kU_k^\top)W^{-1}$. Explicit
orthogonalization gives $\operatorname{rank}(I-T_k)=k$ in every retained run.
Across 526 archived ranks, the maximum Gram-matrix error is
$2.4\times10^{-15}$. Thus ``rank $k$'' here means both $k$ iterations and
cumulative edit rank $k$ for these scalar runs. In the multivariate full-QR
construction, one iteration may add several independent directions, so we report
iteration count and cumulative edit rank separately.

For descriptive stopping only, a rank is selected if its validation group-bootstrap
AUROC interval lies wholly in $[.45,.55]$ for three consecutive ranks. The same
validation split also selected layer and intact $C$, and the $n=200$ audit has
essentially no power under this equivalence rule. We therefore do not use selected
counts as numerical estimates in the main argument. To test post-edit accessibility
more strongly, a separate analysis divides validation groups into selection and
audit halves, selects $C\in\{10^{-4},10^{-3},10^{-2},10^{-1},1\}$ independently
at every rank, swaps the halves, and repeats over five split seeds.

The real-feature maps are fixed before evaluation. The primary dense family is
$A_\kappa=Q D_\kappa Q^\top$, where $Q$ is a randomized signed-Hadamard basis and
the 1,024 diagonal entries of $D_\kappa$ are log-spaced from
$\kappa^{-1/2}$ to $\kappa^{1/2}$. Five fixed maps are used per condition. Every
post-edit feature is inverse-mapped before attack, and all probes are fit in the
original training-standardized coordinates with the intact validation-selected
$C=.01$, which is neither reselected nor transported across maps.

A preprocessing ablation recomputes coordinatewise training mean and standard
deviation after each dense map at
$\kappa\in\{1,2,3,5,10,100,1000\}$. Erasure is performed in those restandardized
mapped coordinates; edited features are then unstandardized and inverse-mapped
before the common attacker. We retain the pre- and post-standardization condition
numbers and all validation predictions. At rank ten, 2,000 paired group-bootstrap
replicates resample the 193 validation source videos and compare the mean of the
five fixed mapped trajectories with identity. A secondary family uses 20 positive
diagonal maps with randomly permuted log-spaced scales; coordinatewise
restandardization cancels that family algebraically. Twenty Haar orthogonal maps
provide a solver sanity check. Unshrunk empirical covariance is evaluated on five diagonal maps
per condition with no floor or shrinkage. These declared map families are
existential stress tests, not a distribution over all invertible maps.

A stronger dense-map analysis removes the fixed-attacker concern without using
official validation or test. For each of five source-video group partitions of the
1,600 official-training examples, role A (793--801 examples) learns the erasure
trajectory and both original and mapped coordinate standardization; role B
(484--490) trains a fresh post-edit attacker; and role C (310--318) is evaluated
once. B is internally group-split into fit and tuning roles, selects
$C\in\{10^{-3},10^{-2},10^{-1},1\}$ independently for every map and audited rank,
then refits on all B. We compare identity with the same five $\kappa=10$ dense maps
at ranks 0/1/2/5/10. Every A/B/C source-group overlap is zero. The 25 map--split
differences describe sensitivity over the declared five maps and five partitions;
their empirical range is not labeled a population confidence interval.

The disclosed research log contains 1,865 historical official-test configurations,
and no untouched local holdout remains. Those results are descriptive and absent
from the primary tables. The visual identification analyses are exploratory:
they use official validation or three-way roles drawn from official training and
record zero official-test queries, but validation reuse precludes confirmatory
language. A protocol timeline documents every split role, and a prospective
replication protocol is frozen for a future dataset.

\subsection{Closed-Form and Cross-Fitted Baselines}

We directly implement Euclidean Mean Projection (equivalent here to full-rank SAL)
and the exact affine LEACE map under empirical and OAS covariance. Estimating the
class-mean difference and training a logistic attacker on the same sample forces
zero training gradient at zero coefficient after mean equalization. Its exact
.500 test AUROC is therefore only an algebraic sanity check, not evidence of a
population one-direction route.

Our cross-fit calibration separates data roles within official training. For each
of five source-video-disjoint partitions, A estimates an intervention from a
balanced subsample, B trains a fresh attacker, and C evaluates it. We vary the A
sample size, report split dispersion, and query neither official validation nor
test. The synthetic version draws 14,000 samples independently from
$X\sim\mathcal N(0,I_{1024})$ with $Y=\mathbb 1[X_1\geq0]$: 10,000 form the
eraser pool, 2,000 train a fresh logistic attacker with $C=.01$, and 2,000 are
evaluated once. Coordinates are standardized using the eraser and attacker pools.
For each balanced A-sample, sample MP/SAL projects out its normalized empirical
class-mean difference; the population oracle instead zeros $X_1$. Ten independent
seeds are run. This distinguishes finite-sample residual access from a claim about
population guarding rank.

Controls remove 20 Gaussian and 20 variance-matched random subspaces, top-PCA
subspaces, and five sequential shuffled-label subspaces in the identical Euclidean
or OAS metric space. We retain AUPRC, representation displacement, and raw
predictions. Stability comparisons use one common standardization and whitening
map fit on the official training split. Split-specific whitening is never compared
with Euclidean principal angles.

\section{Results}

\paragraph{Evidence hierarchy.}
Theorems~\ref{thm:ridge-count} and~\ref{thm:full-qr-count} establish exact
population non-identification. The finite circular experiment asks whether the
separation survives a practical optimizer and held-out stopping rule. The visual
study asks only whether coordinate sensitivity is observable in one frozen-feature
case study; it is neither needed for the theorems nor a direct replication of the
motivating video study.

\subsection{Conventional Probes Localize Access}

Before any erasure, validation-only layer sweeps establish where linear contact
access is strongest under the original probe protocol. Table~\ref{tab:layers}
reports the layers fixed for the later audit. V-JEPA2 selects layers 21 and 18;
DINOv2 selects layers 12 and 11. This localization remains useful: it identifies
which representations are audited. It does not make the selected depth a unique
contact layer or a model-independent semantic coordinate.

\refstepcounter{table}
\label{tab:layers}
\begin{center}
\small
\begin{tabular}{@{}llrr@{}}
\toprule
Dataset & Encoder & Layer & Val. AUROC \\
\midrule
100DOH & V-JEPA2 & 21 & .885 \\
TouchMoment & V-JEPA2 & 18 & .944 \\
100DOH & DINOv2 & 12 & .858 \\
TouchMoment & DINOv2 & 11 & .807 \\
\bottomrule
\end{tabular}
\par\smallskip
\parbox{0.97\columnwidth}{\small\textbf{Table~\thetable:}
Validation-selected layers from the initial three-seed sweeps. These values
select audit points; they are not estimates of semantic depth.}
\end{center}

\subsection{The Integer Count Itself Is Not Identified}

Theorem~\ref{thm:ridge-count} supplies the integer statement: identity mixing stops after one
population edit, while every nonzero shear stops after two for ridge
$\lambda\in\{0,.1,1,10\}$. Labels, sufficient dimension, and minimum guarding
rank are unchanged. Theorem~\ref{thm:full-qr-count} applies the motivating full
coefficient-subspace QR rule: for two outputs, cumulative edit rank is two at
$a=0$ and four after any nonzero shear, despite a common sufficient dimension and
minimum guarding rank of two. Figure \ref{fig:identification} distinguishes these
estimands and gives a complementary one-direction-at-a-time result. With sufficient dimension and
minimum guarding rank two, Euclidean count is two under orthogonal coordinates
and eight for all 50 maps at each $\kappa\in\{3,10,100,1000\}$. The covariance metric
returns two in every trial.

The finite circular bridge reaches the same conclusion under the published
held-out stopping rule. At $n=4{,}000$, identity mixing gives $K=1$ in all 20 runs,
as does the weak shear $a=.25$. For each tested shear
$a\in\{.5,.75,1,1.25,2\}$, all 20 runs accept at least two updates. At $a=1$,
post-first-edit held-out $R^2$ is .195 and circular MAE is
$59.5^\circ$ on average; capped mean $K$ is 5.75 (range 3--8), with two of 20
runs right-censored at eight. Under the literal sequential recurrence,
$\operatorname{rank}(I-T_K)$ has reached four while
$\operatorname{rank}(T_K)$ remains about two, so $2K$ is not an independent
removed-subspace dimension. The decisive identity/$a=1$ separation also holds in all ten
runs at minibatch sizes 64 and 256, with mean shear counts 5.4 and 4.9. At
$n=1{,}000$, only 15/20 $a=1$ runs accept a second update, consistent with the
continuous residual magnitude in Equation~\ref{eq:shear-residual}; finite stopping
is sample- and optimizer-relative rather than a noiseless copy of the exact
theorem.

The real-feature test uses V-JEPA2 layer 21 on the 100DOH official validation
split and isolates the same dependence without changing images, labels, or the
class of intact linear scores. After mapped-coordinate restandardization, mean
rank-ten AUROC is .599 at $\kappa=1$ and .666/.673/.745/.791 at
$\kappa=2/3/5/10$, then .853/.840 at $100/1{,}000$. The corresponding mean
post-standardization condition numbers are 2.09/3.19/5.41/10.89 and
110.64/1,070.68. Source-video-paired 95\% bootstrap intervals for the mean
mapped-minus-identity rank-ten difference are [.008,.134] already at $\kappa=2$
and [.110,.275] at $\kappa=10$ (Figure~\ref{fig:identification}, bottom left).
Without restandardization, the previously declared dense maps give
.802/.841/.853 at $\kappa=10/100/1{,}000$; the practical effect is therefore not
explained by coordinate scaling alone. The secondary diagonal family gives
.599/.809/.848/.846 at $\kappa=1/10/100/1{,}000$ and is algebraically canceled by
coordinatewise restandardization.
Unshrunk empirical-covariance intervention, with no shrinkage or floor, gives the
identical complete trajectory for all maps: rank-one .587 and rank-ten .518, with
maximum prediction discrepancy $2.11\times10^{-6}$. Thus unshrunk empirical covariance matches
the corollary of Proposition~\ref{prop:metric-equivariance}. Twenty Haar
orthogonal maps also reproduce the complete
Euclidean AUROC trajectory, with maximum prediction discrepancy
$4.27\times10^{-6}$. Anisotropy, not arbitrary coordinate renaming or solver noise,
drives the stress-test difference.

\begin{figure*}[t]
\centering
\includegraphics[width=0.96\textwidth]{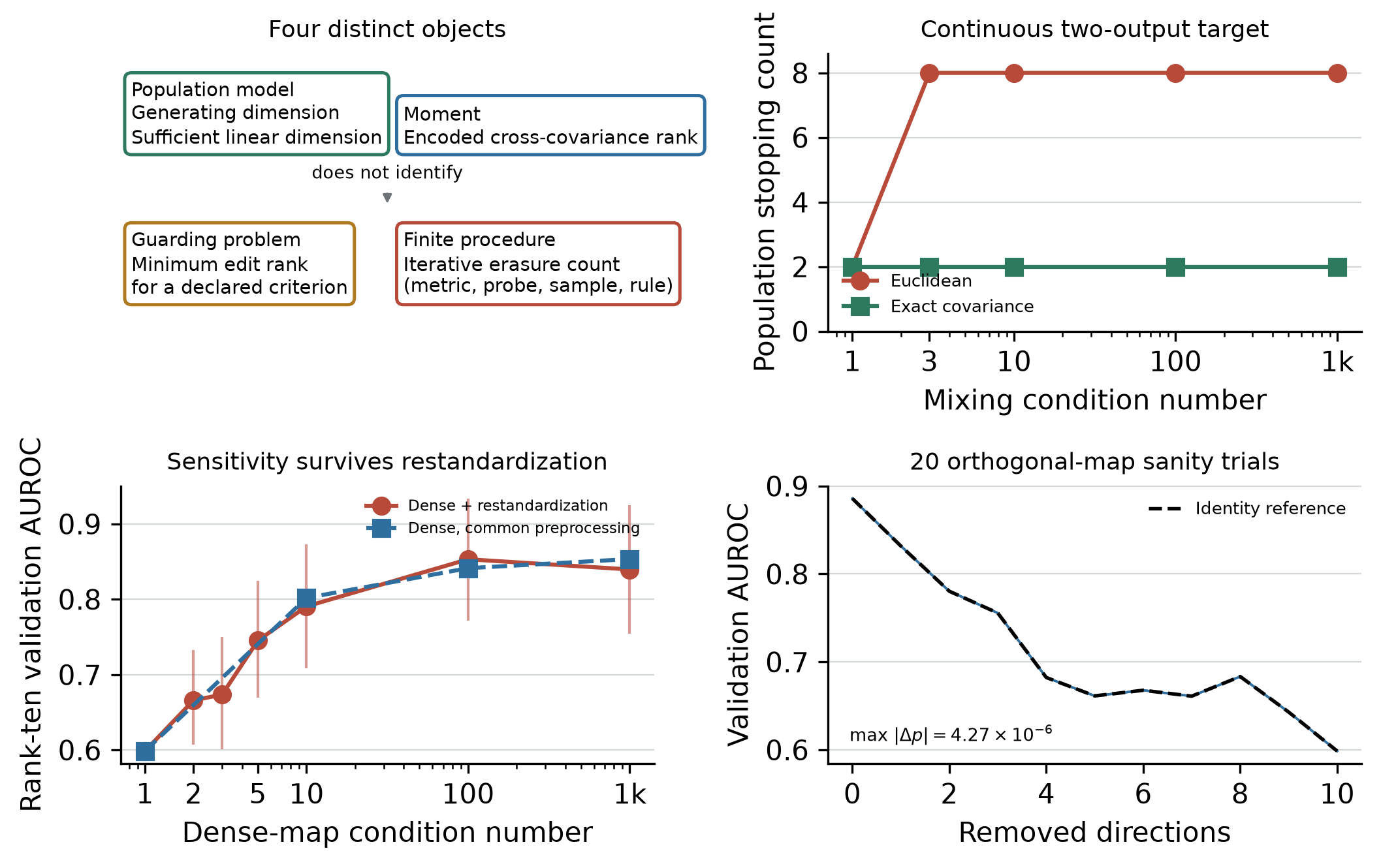}
\caption{\textbf{Identification and controls.} \emph{Conceptual, top left:} four
quantities that need not agree. \emph{Population, top right:} a continuous
two-output target has Euclidean count two or eight, while the covariance-metric
count stays at two. \emph{Real features, bottom:} dense-map sensitivity
survives mapped-coordinate restandardization (left); vertical intervals are paired
source-video bootstrap intervals for mean mapped-minus-identity AUROC. Bottom
right: 20 orthogonal maps reproduce the identity trajectory, and individual
trials overlap.}
\label{fig:identification}
\end{figure*}

\subsection{Stronger Attackers Preserve Procedure Dependence}

The official-training-only dense-map audit directly retunes a fresh attacker for
each map and rank. Rank-zero AUROC is identical at .822. At rank ten, identity
averages .665 across five source-disjoint partitions, whereas the five
$\kappa=10$ maps average .776 across 25 map--partition runs. The paired difference
has mean +.111 and empirical range [.075,.155], and is positive for all 25 pairs.
The corresponding differences are already positive for all pairs at rank one
(mean +.033) and rank two (+.060). Thus the transformed-map effect is not an
artifact of reusing one post-edit $C$ or one attacker fit. These are conditional
within-training case-study results, not population intervals.

Figure \ref{fig:attackers} retunes regularization at every rank on one
group-disjoint validation half and evaluates on the other, with halves swapped
over five seeds. On 100DOH, Euclidean/OAS AUROC is .827/.586 after one edit and
.598/.494 after ten. On TouchMoment it is .924/.811 after one and .778/.447 after
ten. The curves remain strongly metric-dependent and nonmonotone even when each
rank receives a fresh attacker search. They should not be compressed into a
precise semantic count. We call the reported quantity orientation-fixed
concordance: it is ordinary AUROC with score orientation fixed on the selection
half and never flipped on the audit half. A value below .5 therefore records an
orientation reversal on the audit half, not equivalence to chance. Flipping it
after inspection would turn the same instability into an apparently above-chance
score.

\begin{table*}[t]
\caption{Controlled identification results. Count and cumulative edit rank are
distinguished in the theorem statements and finite recurrence.}
\label{tab:population-results}
\centering
\scriptsize
\begin{tabular}{@{}p{0.25\textwidth}p{0.12\textwidth}p{0.15\textwidth}p{0.10\textwidth}p{0.25\textwidth}@{}}
\toprule
Construction & Setting A & Setting B & Difference & Identified conclusion \\
\midrule
Population ridge shear count & $a=0$: 1 & $a\ne0$: 2 & +1 & Same rank-one concept under isotropic or transported ridge \\
Continuous target, Euclidean count & $\kappa=1$: 2 & $\kappa\ge3$: 8 & +6 & Same sufficient dimension and guard rank 2 \\
Continuous target, exact-cov. count & $\kappa=1$: 2 & $\kappa=1000$: 2 & 0 & Affine-equivariant population control \\
Finite circular QR, $n=4000$ & $a=0$: $P(K\ge2)=0$ & $a=1$: $P(K\ge2)=1$ & 0/20 vs. 20/20 & Published Adam/QR threshold remains coordinate-sensitive \\
\bottomrule
\end{tabular}
\end{table*}

\begin{table*}[t]
\caption{Empirical case-study results; every row makes zero official-test queries.
For dense restandardization, brackets give a source-video-paired 95\% bootstrap
interval conditional on fixed maps and selections. For fresh attackers they give
the empirical range over 25 declared map--partition pairs, not a confidence
interval. Cross-fit $n_A$ denotes eraser-estimation sample size within official
training.}
\label{tab:empirical-results}
\centering
\scriptsize
\begin{tabular}{@{}p{0.25\textwidth}p{0.12\textwidth}p{0.15\textwidth}p{0.10\textwidth}p{0.25\textwidth}@{}}
\toprule
Audit & Setting A & Setting B & Difference & Identified conclusion \\
\midrule
Dense + restandardization, rank 10 & $\kappa=1$: .599 & $\kappa=10$: .791 & +.192 [.110,.275] & Coordinate sensitivity survives ordinary preprocessing \\
Dense fresh attacker, rank 10 & $\kappa=1$: .665 & $\kappa=10$: .776 & +.111 [.075,.155] & Map-specific per-rank tuning; 25/25 paired differences positive \\
Unshrunk empirical-covariance rank 10 & $\kappa=1$: .518 & $\kappa=1000$: .518 & .000 & Transported-metric corollary \\
Orthogonal maps, all ranks & Identity & 20 maps & $|\Delta p|\le4.27\times10^{-6}$ & Euclidean-invariant sanity check \\
Synthetic cross-fit MP/SAL & $n_A=100$: .881 & $n_A=8000$: .572 & -.309 & Sample error despite rank-one population guard \\
100DOH-max cross-fit LEACE-OAS & $n_A=100$: .876 & $n_A=6000$: .647 & -.229 & Cross-fit within official training \\
TouchMoment cross-fit LEACE-OAS & $n_A=100$: .927 & $n_A=4000$: .626 & -.301 & Cross-fit within official training \\
\bottomrule
\end{tabular}
\end{table*}

\begin{figure*}[t]
\centering
\includegraphics[width=0.91\textwidth]{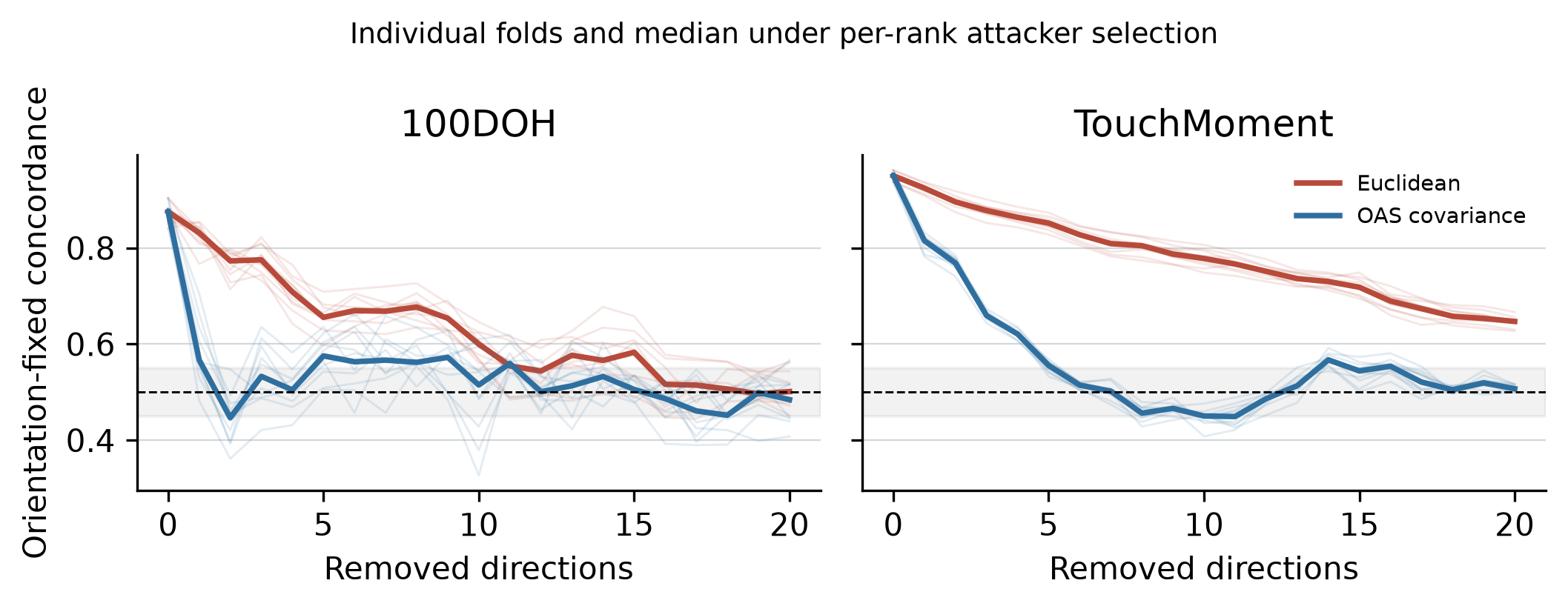}
\caption{\textbf{Per-rank attacker audit.} Each rank selects logistic
regularization on one validation-group half and evaluates on the other, with
halves swapped over five seeds. Thin lines are all ten folds, while thick lines are
medians. These are not confidence bands.}
\label{fig:attackers}
\end{figure*}

\subsection{Secondary Estimation Calibrations}

Covariance-subsampling and cross-fit learning curves support interpretation rather
than the central theorem. They show that covariance geometry remains
sample-dependent and that a sample-estimated rank-one edit can leave substantial
fresh-attacker access even for a known population rank-one guard. Appendix
\ref{app:estimation-calibration} reports the full curves and fold dispersion; the
results are not evidence for multiple population contact routes.

\section{Limitations}

Our affine theorem and latent construction concern linear representations and
linear interventions. They do not identify nonlinear, tokenwise, spatial,
temporal, or causal concept organization. The visual study centers on contact and
one primary V-JEPA2 checkpoint. Secondary predicates share the same 100DOH images.
Their results, nuisance controls, random baselines, and temporal audits remain in
the artifact. DINOv2 is an image encoder, not a temporally matched video model.
Theorem~\ref{thm:full-qr-count} matches the motivating output width, squared loss,
and full-QR removal algebra at population level. The finite circular bridge also
matches the published angular output, Adam settings, epochs, literal QR recurrence,
and held-out thresholds, but it uses a controlled distribution rather than the
motivating video activations and does not reproduce steering. Because the source
does not report minibatch size, we declare 128 and audit 64/256 rather than claim an
exact hidden implementation match. Its fixed held-out split is repeatedly queried
across ranks, as in the published specification, so it is a procedure stress test
rather than a new generalization estimate. The ambient-eight leading-vector
experiment remains a complementary sequential variant.

The contact labels are imperfect semantic targets. In 100DOH, released object-box
presence is nearly label-deterministic, and the nuisance-matched hard subset has
only 58 pairs. TouchMoment contrasts onset with a nearby pre-onset frame, so hand
closure, approach, endpoint appearance, and action phase remain informative.
Temporal and nuisance controls delimit these effects but cannot make either task a
pure test of physical touch.

Cross-fitting diagnoses sample-estimated generalization, not universal linear
guardedness. Its known-rank-one calibration shows that residual access can be
large without multiple population routes. OAS and Ledoit--Wolf reduce covariance
estimation problems but do not create a canonical semantic metric. Descriptive
stopping reuses validation for layer, regularization, and rank and lacks formal
sequential error control. Historical official-test exposure was extensive, no new
untouched holdout is available, and none of those results is confirmatory. A
prospective protocol is frozen for a future dataset. The anisotropic visual maps
comprise five fixed dense signed-Hadamard SPD maps and a secondary diagonal family.
They establish dependence for those declared families, not for a distribution over
all invertible maps. Dense-map sensitivity survives mapped-coordinate
restandardization and paired source-video resampling, but it remains an exploratory
official-validation result. The fresh-attacker map audit remains conditional on five
declared training partitions and five maps; its 25 paired values are not independent
population draws. Transported-metric equivariance requires an equivariant probe,
transported positive-definite regularization after rank loss, common support, and
equivariant tie-breaking. Rank-deficient Moore--Penrose fitting is explicitly not
covered. Exact covariance is only one corollary and the unshrunk empirical-
covariance trajectory is an
algebraic consistency control, not evidence that covariance is the semantically
correct metric. Finally,
an invertible map preserves information and
linear prediction classes, not Euclidean geometry. Readers may legitimately study
that geometry after declaring it as part of the estimand.

\paragraph{Reproducibility.}
The supplementary artifact contains annotation-safe manifests, extraction
configurations, every analysis script, spectra, fold-level and aggregate results,
and a one-command cached-analysis entry point. Licensed media, feature caches, and
large prediction arrays are not redistributed.

\section{Conclusion}

Many erased directions need not mean many semantic dimensions. Population ridge
count changes from one to two for the same rank-one concept. Under full two-output
QR removal, cumulative edit rank changes from two to the ambient dimension four, while
a complementary sequential construction reaches ambient dimension eight. Under
the published finite circular Adam/QR protocol, identity stops after one update in
all large-sample runs while shears pass multiple held-out updates and can continue
after edit rank saturates. We prove the conditions under which any covariantly
transported positive-definite metric makes the complete trajectory
affine-equivariant, with exact covariance as a corollary. The same frozen visual
features show a large dense-map change after coordinatewise restandardization and
map-specific fresh-attacker tuning, while orthogonal maps reproduce the identity
trajectory. Finite-sample calibration further prevents residual attacker access
from being read as multiple population routes. None of these results estimates a
contact dimension.

The defensible object is an explicitly named procedure-relative estimand. Reports
of iterative erasure count should specify the metric, probe loss, regularizer, covariance
estimator, sample, orientation rule, stopping rule, and test-use policy, and should
separate generating dimension, sufficient dimension, encoded cross-covariance
rank, and guarding rank. Without those commitments, iterative erasure count is
not an affine-invariant concept dimension.

\bibliographystyle{plain}
\begingroup
\small
\bibliography{refs}
\endgroup

\clearpage
\appendix
\setcounter{table}{0}
\renewcommand{\thetable}{A\arabic{table}}
\renewcommand{\theHtable}{A\arabic{table}}
\section{Supplementary Finite Circular-QR Results}

Table~\ref{tab:finite-qr-supp} gives representative sample-size cells from the
primary minibatch-128 campaign. $K$ counts accepted QR updates before the first
held-out failure. Runs still passing at the eight-update cap are right-censored;
$K=8$ can also be observed exactly. The $n=500$ identity row shows why finite error must be
calibrated: even the untransformed rank-two construction sometimes passes a second
probe. As sample size grows, identity stabilizes at one while moderate shears pass
multiple updates. This is evidence about the declared procedure, not a population
dimension estimate.

\refstepcounter{table}
\label{tab:finite-qr-supp}
\begin{center}
\centering
\scriptsize
\resizebox{\columnwidth}{!}{%
\begin{tabular}{@{}rrrrrr@{}}
\toprule
$n$ & Shear $a$ & Capped mean $K$ & $P(K\geq2)$ & Censored/20 & Mean post-first $R^2$ \\
\midrule
500 & 0 & 1.80 & .45 & 0 & .089 \\
500 & .5 & 1.95 & .40 & 0 & .114 \\
500 & 1 & 1.85 & .40 & 0 & .116 \\
500 & 2 & 1.10 & .25 & 0 & .065 \\
1000 & 0 & 1.00 & .00 & 0 & .001 \\
1000 & .5 & 1.75 & .35 & 0 & .079 \\
1000 & 1 & 2.95 & .75 & 0 & .122 \\
1000 & 2 & 2.35 & .45 & 1 & .109 \\
4000 & 0 & 1.00 & .00 & 0 & -.002 \\
4000 & .5 & 4.35 & 1.00 & 0 & .134 \\
4000 & 1 & 5.75 & 1.00 & 2 & .195 \\
4000 & 2 & 5.90 & 1.00 & 1 & .132 \\
\bottomrule
\end{tabular}}
\par\smallskip
\parbox{0.97\columnwidth}{\scriptsize\textbf{Table~\thetable:} Finite
circular-target stress test, 20 seeds per cell. Some small-$n$
shears fail even initially, so post-first means use runs that accepted the intact
probe. Full ranges, censoring rates, and all seven shears are retained in the
artifact.}
\end{center}

For the seven $n=4{,}000$ conditions
$a\in\{0,.25,.5,.75,1,1.25,2\}$, the numbers of right-censored runs are
$0,0,0,1,2,3,1$ out of 20. Table~\ref{tab:finite-survival} therefore reports the
empirical survival probabilities through the declared cap rather than relying on
capped means alone.

\refstepcounter{table}
\label{tab:finite-survival}
\begin{center}
\small
\begin{tabular}{@{}rrr@{}}
\toprule
$k$ & $P(K\geq k)$, $a=0$ & $P(K\geq k)$, $a=1$ \\
\midrule
1 & 1.00 & 1.00 \\
2 & .00 & 1.00 \\
3 & .00 & 1.00 \\
4 & .00 & .90 \\
5 & .00 & .85 \\
6 & .00 & .50 \\
7 & .00 & .30 \\
8 & .00 & .20 \\
\bottomrule
\end{tabular}
\par\smallskip
\parbox{0.97\columnwidth}{\small\textbf{Table~\thetable:}
Censoring-aware finite circular-QR summary for the decisive $n=4{,}000$
conditions. Probabilities through $k=8$ are identified despite right-censoring at
the cap; behavior beyond eight is not.}
\end{center}

At $n=4{,}000$ and $a=1$, minibatch-size sensitivity preserves the separation.
For sizes 64/128/256, identity has mean $K=1.0$ and $P(K\geq2)=0$; the shear has
mean $K=5.4/5.75/4.9$ and $P(K\geq2)=1$. Sizes 64 and 256 use ten seeds each; the
primary size 128 uses twenty.

\section{Supplementary Estimation Calibrations}
\label{app:estimation-calibration}

The raw positive-spectrum condition numbers are $9.1\times10^4$ for full
100DOH-max and $3.1\times10^5$ for TouchMoment. Flooring caps the post-floor value
at exactly $10^4$. Earlier larger diagnostics were computed before flooring.
Repeated subsamples show that covariance sample composition does not explain away
the trend (Figure \ref{fig:estimation}, left). OAS rank-one validation AUROC on
100DOH-max has median .838 and range [.830,.845] at $n=200$, versus .639 at all
17,206 samples. TouchMoment has median .844 and range [.828,.861] at $n=500$,
versus .722 at all 11,564. Individual subsamples are shown rather than treating
five repetitions as a confidence interval. More samples improve an estimator, but
they do not create a canonical semantic metric or rank.

The synthetic cross-fit calibration changes the interpretation of fresh-attacker
access. In 1,024 dimensions, sample-estimated MP/SAL leaves independent AUROC
.881, .754, .683, .622, and .572 at eraser sample sizes 100, 1,000, 2,000, 4,000,
and 8,000, despite the known population rank-one guard. The population oracle
gives .499 at every size (Figure \ref{fig:estimation}, center). Large residual
access is therefore compatible with high-dimensional estimation error alone.

Cross-fitting within official training shows the same sample-size dependence
(right). LEACE-OAS falls from .876 to .647 over $n_A=100$--6,000 on 100DOH-max
and from .927 to .626 over $n_A=100$--4,000 on TouchMoment. Five
source-video-disjoint partitions give standard deviations below .014 at the
largest sizes. These results do not identify population guarding rank or multiple
contact routes. They show how much independently trained access survives a
sample-estimated edit under a declared A/B/C protocol.

\begin{figure*}[t]
\centering
\includegraphics[width=0.98\textwidth]{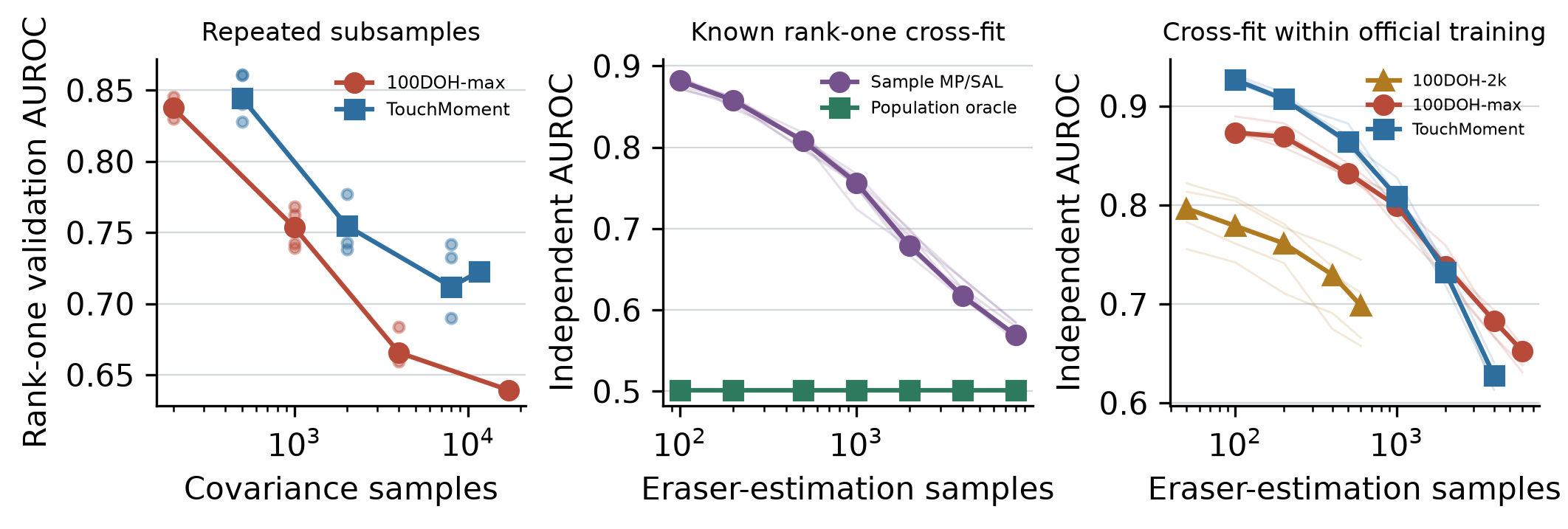}
\caption{\textbf{Estimation calibration.} Left: repeated OAS covariance
subsamples. Center: sample-estimated rank-one edits leave access for a known
population rank-one concept, unlike its oracle. Right: cross-fitting within
official training. Thin lines and points show every repetition, while thick lines
are medians.}
\label{fig:estimation}
\end{figure*}
\FloatBarrier

\section{Optional Supplementary Native-Basis Channel Case Study}
\label{app:probing}

This optional study is not part of the paper's affine non-identification evidence.
We revisit the channel result with AUROC/AUPRC rather than thresholded F1. Within
official training only, A ranks channels, B tunes and fits an attacker, and C
evaluates; all roles are source-video disjoint. We rank by orientation-free
univariate AUROC, replace selected standardized coordinates by the A-sample mean,
compare uniformly random sets, and score both the intact B attacker and a fresh
post-edit B attacker. Primary selected-layer audits use 20 outer splits. Focused
depth and full-100DOH calibrations use 10 and five splits, respectively. Layers
19--20 use a source-stratified 2,000-train/618-validation subset with no test rows;
their A/B/C channel analysis uses training only.

\refstepcounter{table}
\begin{center}
\scriptsize
\resizebox{\columnwidth}{!}{%
\begin{tabular}{@{}llrrrrrr@{}}
\toprule
Data & Encoder & Layer & Splits & Fixed $\Delta$ & Fresh $\Delta$ & Jacc. & 407 top-50 \\
\midrule
100DOH-2k & V-JEPA2 & 21 & 20 & .017 & .003 & .478 & 0/20 \\
100DOH-full & V-JEPA2 & 21 & 5 & .032 & .000 & .713 & 0/5 \\
TouchMoment & V-JEPA2 & 18 & 20 & .052 & .001 & .841 & 0/20 \\
Touch-subset & V-JEPA2 & 19 & 10 & .022 & .002 & .605 & 0/10 \\
Touch-subset & V-JEPA2 & 20 & 10 & .038 & .005 & .639 & 0/10 \\
TouchMoment & V-JEPA2 & 21 & 20 & .068 & .000 & .732 & 0/20 \\
TouchMoment & V-JEPA2 & 22 & 20 & .062 & .001 & .800 & 20/20 \\
TouchMoment & V-JEPA2 & 23 & 10 & .044 & .001 & .735 & 10/10 \\
TouchMoment & V-JEPA2 & 24 & 10 & .065 & .001 & .793 & 0/10 \\
100DOH & DINOv2 & 12 & 20 & .008 & .001 & .481 & -- \\
TouchMoment & DINOv2 & 11 & 20 & .074 & .007 & .800 & -- \\
\bottomrule
\end{tabular}}
\par\smallskip
\parbox{0.97\columnwidth}{\scriptsize\textbf{Table~\thetable:} Excess AUROC drop of
univariate-AUROC-ranked top-50 sets relative to matched random sets. ``Fixed''
scores the intact attacker; ``Fresh'' retrains after editing. Jaccard is mean
pairwise top-50 overlap across outer splits.}
\end{center}

The table separates stable routes from necessary information. Full TouchMoment
sets are highly repeatable and perturb fixed readers by .044--.068 AUROC beyond
random, while fresh-attacker excess is at most .0014. The focused layer-19/20
subset gives fixed excess .022/.038 and fresh excess .0015/.0049. On all 21,290 balanced
100DOH images, more data raises top-50 Jaccard to .713 and the fixed effect to
.032, while fresh excess remains .0004. DINOv2 shows the same qualitative
fixed-versus-fresh separation, with a modest .0068 TouchMoment residual.

Channel 407 is top-50 in 20/20 layer-22 and 10/10 layer-23 splits, with median
ranks 22.5 and 15.5, but not at layers 18--21 or 24. Recurrence is not necessity:
masking 407 alone at layer 22 drops fixed AUROC by .00059 versus .00061 for a
matched random channel; at layer 23 its excess fixed drop is .00049. The most
frequently top-ranked channels at layers 21--24 are also individually weak; the
largest tested excess is .00119, and all fresh-attacker excesses are effectively
zero.

Finally, opposite-label donor patching provides a set-level sensitivity control.
On full TouchMoment at layers 18, 22, and 23, selected top-50 patches yield fixed AUROC
.686/.680/.699, versus .890/.863/.856 for random-channel patches and
.898/.852/.868 for same-label selected-channel patches. Focused layer 20 gives
.711/.882/.905, and full 100DOH gives .697/.852/.898. Donor selection is
label-informed, so retrained patch attackers
answer a new intervention-induced prediction problem; only fixed selected-versus-
control contrasts support the route claim.

The historical hard-zero experiments motivated this audit but are not needed for
its conclusion. The corrected result is narrower and stronger: channel 407 is a
reproducible member of a late-layer ensemble, while no tested individual axis is
necessary and ensemble information remains accessible to fresh linear attackers.
This analysis is axis-aligned in one fixed encoder basis and is not an
affine-invariant decomposition. Mean replacement acts on saved mean-pooled
representations rather than downstream network computation, and donor patching is
label-informed. Those restrictions are why the result remains supplementary.

\end{document}